\documentclass{article}

\usepackage{graphicx}
\usepackage{booktabs}
\usepackage{hyperref}
\usepackage{url}
\usepackage{amsfonts}
\usepackage{xcolor}
\usepackage{amsmath}
\usepackage{amssymb}
\usepackage{amsthm}
\usepackage[shortlabels]{enumitem}
\usepackage{algorithm}
\usepackage{algorithmic}

\usepackage[preprint]{icml2026}

\usepackage[capitalize,noabbrev]{cleveref}

\theoremstyle{plain}
\newtheorem{theorem}{Theorem}[section]
\newtheorem{proposition}[theorem]{Proposition}
\newtheorem{lemma}[theorem]{Lemma}

\theoremstyle{definition}

\theoremstyle{remark}
\newtheorem{remark}[theorem]{Remark}

\icmltitlerunning{Ricci-Filtration for RAG Reranking}

\begin{document}

\twocolumn[
  \icmltitle{Ricci-Filtration:  Boosting Retrieval-Augmented Generation Reranker to Query-Answer Tasks by Discrete Ricci Flow}



  \icmlsetsymbol{equal}{*}

  \begin{icmlauthorlist}
    \icmlauthor{Tian Qin}{xxx}
    \icmlauthor{Wei-Min Huang}{yyy}

  \end{icmlauthorlist}

  \icmlaffiliation{xxx}{Math department, Lehigh University}
    \icmlaffiliation{yyy}{Math department, Lehigh University}

  \icmlcorrespondingauthor{Tian Qin}{tiq218@lehigh.edu}
  \icmlcorrespondingauthor{Wei-Min Huang}{wh02@lehigh.edu}

  \icmlkeywords{Machine Learning, ICML}

  \vskip 0.3in
]



\printAffiliationsAndNotice{}  

\begin{abstract}
Ricci flow \cite{ricciflow} is a curvature-guided diffusion process that deforms space by shrinking regions of high positive curvature and expanding those with negative curvature. Similarly, discrete Ricci flow on weighted graphs \cite{Ni2019CommunityDO} modifies edge weights by shrinking edges with positive Ricci curvature and stretching those with negative Ricci curvature, effectively increasing the separation between clusters. Inspired by these two cornerstone works, 
we propose a geometry-based RAG reranker enhancement procedure called Ricci-Filtration. By modeling the input query and initial retrieved chunks as a network, where the input query and chunks serve as nodes and embedding-based pairwise relations define an initial graph, Ricci-Filtration leverages discrete curvature and Ricci flow to evaluate the structural importance of each chunk with respect to the user query. The system first filters the initial chunks based on their geometric curvature relative to the query; then, a reranker processes the remaining chunks to enhance generative performance. We theoretically prove that normalized discrete Ricci flow can detect community structures by identifying distinct asymptotic behaviors in edge weights. This supports the removal of ``noisy'' document chunks characterized by large weights and negative Ricci curvature relative to the query node. Extensive experiments confirm that Ricci-Filtration outperforms several baseline reranking methods in accuracy, precision, recall, and F1 scores. Furthermore, ablation studies demonstrate that the Ricci-Filtration generally outperforms the baseline under various settings, highlighting the framework's robustness across different architectures.

\end{abstract}

\section{Introduction}
\begin{figure*}[h]
  \includegraphics[width=\linewidth]{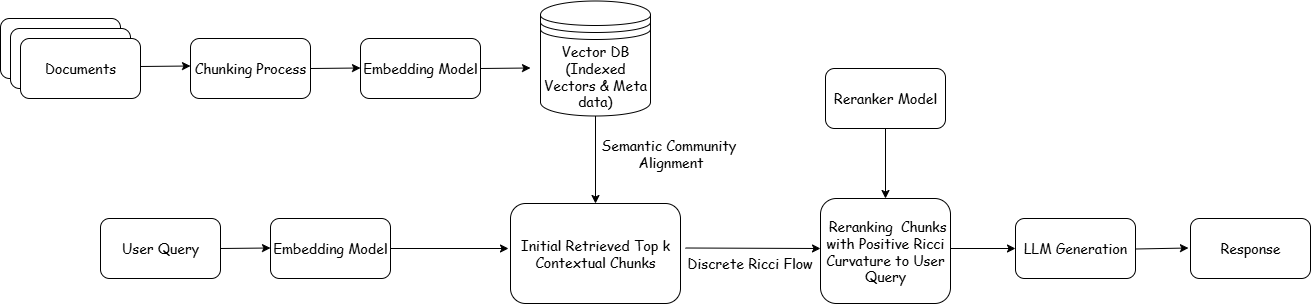}
  \caption{Overview of the Ricci-Filtration framework. Ricci-Filtration is designed to better retrieve relevant information
from stored knowledge by employing discrete Ricci flow on graphs formed by the query and chunks, which eventually leads to a dynamic number of selected chunks and superior generation quality.}
  \label{fig:ricci-rag}
\end{figure*}

Retrieval-augmented generation (RAG) \cite{RAG} is a popular method for using LLMs to answer queries based on data that is too large for the context window of a language model, which means the maximum number of tokens that can be processed by the LLM
at once \cite{liu-etal-2024-lost}. A typical RAG system is designed to retrieve a few records that are specifically relevant to the user's query and collectively small enough to fit within the language model's context window \cite{Baumel2018QueryFA,QALaskar,QAYao}. This process allows the LLM to generate a response informed by specific, external information rather than just its pre-trained knowledge. However, the initial retrieval process may contain many irrelevant chunks, especially in Query-Answer (QA) tasks, which distracts the LLM and leads to hallucinations. More specifically, most RAG systems utilize vector search to select the most related chunks from the knowledge base, while vector search can be easily fooled by keyword overlap or shallow semantic similarity. One solution is to employ rerankers as fine-grained filters to not only move the most relevant ``gold'' chunks to the top, but also filter out ``noisy'' or irrelevant documents.

 The  paradigm described above was first proposed by \citet{glass-etal-2022-re2g}, where Cross-Encoders are employed as rerankers to perform a deep token-level comparison between the query and the initial selected chunks. Unlike bi-encoders where the query and document are mapped into a shared vector space independently, cross-encoders treat the query and document as a single combined input sequence by performing full self-attention across every token in both query and document simultaneously. Examples include transformer-based encoders like BGE-Reranker \cite{bge_embedding} and Sentence-BERT \cite{reimers2019sentencebert}. Since we have to input the query and the document into the model together, the latency can be high. Additionally, most cross-encoders are based on BERT-style architectures, which need to be fine-tuned and have token limits, making them less general in practice. There are also works that use LLMs as rerankers. \citet{sun2023is} demonstrated that LLMs, when prompted to rank documents, significantly outperform smaller cross-encoders. However, the price would be the relatively high computational cost. On the other hand, the current reranker paradigm requires a fixed parameter $k$ to select reranked chunks. Such a one-size-fits-all setting may make the reranker less flexible to the variable relationships between query and text chunks. For instance, if $k$ is too large, the selected reranked chunks can still contain noisy information degrading the quality of the LLM's generation.

To address these issues, we introduce Ricci-Filtration, a novel  RAG enhancement framework featuring a geometric augmented reranking process based on the  discrete Ricci flow \cite{Ni2019CommunityDO}.  From the geometric point of view,  the discrete Ricci flow defined on weighted graphs and deforms edges weights as time progresses. As a result, edges of large positive Ricci curvature will shrink and edges of negative Ricci curvature will be stretched, i.e., separated from each other. Unlike traditional rerankers, Ricci-Filtration relies exclusively on precomputed embeddings and the intrinsic geometric relationships between query and text chunks. By modeling the input query and initial retrieved chunks as a graph (network), where the input query and chunks serve as nodes and embedding-based pairwise relations define the initial graph topology, Ricci-Filtration utilizes discrete curvature and Ricci flow to characterize the structural importance of each node. Specifically, the initial retrieved chunks are first filtered by their geometric curvatures between query, then a reranker is employed to boost the performance of generation.  This approach is LLM-agnostic during the filtration step, since the filtering decision uses the initial embeddings and graph geometry rather than an additional LLM call. Its advantage should therefore be understood as a different accuracy-latency trade-off, not as a uniform cost advantage over cross-encoders or LLM-based reranking. Figure \ref{fig:ricci-rag} illustrates the  framework of Ricci-Filtration. 

The main contributions of our work can be summarized as following:
\begin{enumerate} 
    \item We first model the initial information retrieval from the graph and geometry point of view. We then employ the discrete Ricci flow algorithm to filter the irrelevant information. Note that the number of filtered chunks used in reranker later is now dynamic as the filtration process  totally relies on the intrinsic Ricci curvature between query and text chunks.  To our best of knowledge, this is the first attempt to connect RAG with Ricci flow, a famous curvature guided diffusion process introduced by \citet{ricciflow}.
    \item  To run the discrete Ricci flow on graphs, we construct an initial edge-incidence graph from cosine dissimilarities between the embeddings of the query and the initially retrieved text chunks. Present edges are then assigned positive initial lengths before the flow starts. \citet{Ni2019CommunityDO} developed a python package running discrete Ricci flow on graphs smoothly. Thus the whole implementation of Ricci-Filtration is straightforward and easy to be integrated into existing pipelines.
    \item We give a stylized theoretical analysis showing edge-type separation under normalized Ricci flow. The uniform-neighborhood case $\alpha=0,p=0$ admits a closed-form proof, while the practical choice $\alpha=1/2,p=2$ admits a weaker finite-time separation guarantee on the same symmetric graph family. These results support the filtering mechanism but do not constitute a guarantee for arbitrary embedding-derived retrieval graphs. 
    \item 
    The experiments demonstrate that Ricci-Filtration is superior to many baseline reranker methods on many QA tasks in the sense of accuracy, precision, recall and F1 scores. The ablation studies also show that Ricci-Filtration generally outperforms the baseline under various comparable settings, highlighting the framework's robustness across different architectures.
\end{enumerate}

 The remaining sections of this paper are organized as following schema. In section \ref{preliminary}, we provide some preliminaries  involving concepts of reranker and discrete Ricci flow. In section \ref{Ricci-RAG}, we illustrate the architecture of Ricci-Filtration. The comparisons between RAG with Ricci-Filtration and RAG with other types of rerankers on benchmark datasets are provided  in section \ref{Experiments}. Ablation studies are provided in section \ref{Experiments} as well. In the end,  section \ref{discussion} discusses few characteristics and limitations of Ricci-Filtration and suggests some future directions. The code and instructions needed to reproduce the main experimental results are included in supplementary material.

\section{Preliminary}
\label{preliminary}

\subsection{Reranker in RAG}
 The RAG system consists mainly of two components: Retrieval and Generation. The retrieval process extracts relevant data from external sources via indexing and searching. Indexing organizes documents using inverted indexes for sparse retrieval or dense
vector encoding for dense retrieval  \cite{Gao2023RetrievalAugmentedGF,Khattab2020ColBERTEA,Douze2024TheFL,Yu2024EvaluationOR} to ensure efficiency. The searching phase then leverages these indexes to fetch documents based on user queries, often applying rerankers \cite{Lyu2024CRUDRAGAC,Tang2024MultiHopRAGBR} such as a powerful cross-encoder to refine the ranking of results for better relevance to the user query. The generation component utilizes the retrieved content and user query to formulate coherent
and contextually relevant responses with the prompting and inference phases. LLMs are the preferred performance standard in generation and their dominance is attributed to their ``Emerging'' ability \cite{Wei2022EmergentAO} and recent breakthroughs in following complex human commands  \cite{Ouyang2022TrainingLM}. In this paper, we will focus on improving the reranking step which is a crucial second-stage   in the retrieval process. In general, a reranker acts as a high-precision filter that double-checks the documents retrieved by our initial search before sending them to the generation model. The common choices of rerankers involve cross-encoders and LLM-based rerankers. Cross-encoders process the query and the document simultaneously. Because they can perform ``full attention'' across both texts, they are highly accurate but computationally expensive \cite{Reimers2019SentenceBERTSE,Devlin2019BERTPO}. On the other hand, LLM-based Rerankers use full LLM (like GPT-4o or Gemini) to rank retrieved documents. For example, we can prompt the LLM with the query and a list of docs, asking it to identify which are most relevant \cite{LLMreranker,Yu2024RankRAGUC}. However, the latency and cost will be extremely high. 


\subsection{Discrete Ricci Flow}
\label{discrete ricci flow}

 Drawing on the geometric foundations established by Gauss and Riemann over 150 years ago, the Ricci flow utilizes curvature to provide a quantitative measure of how a space bends at any given point. In  classical geometry, this curvature dictates the distribution of space: areas with high positive curvature are more ``densely packed'', whereas regions of negative curvature tend to spread out or expand. To locate these regions of large curvature, \citet{ricciflow} introduced a curvature guided diffusion process, called the Ricci flow, that deforms the space in a way formally analogous to the diffusion of heat. Under the Ricci flow, regions in a space of large positive curvature shrink to points whereas regions of very negative curvature spread out. See Appendix \ref{the ricci flow} for details. By considering a network (graph) as a discrete counterpart of a manifold and connected sum components as communities, \citet{Ni2019CommunityDO} introduced a discrete Ricci flow on networks for identifying communities in a network. Building on these foundations, we explore the application of discrete Ricci flow to the graphs established between user queries and retrieved document chunks in RAG.

To better illustrate the idea of discrete Ricci flow on weighted graphs, we follow the notations in \citet{Ni2019CommunityDO} and  start with the definition of discrete (Ollivier) \footnote{While the literature also includes Forman-Ricci curvature \cite{Sreejith2016FormanCF}, which easier and faster to compute in large-scale networks, it lacks the geometric depth of Ollivier-Ricci curvature. Consequently, this work focuses exclusively on Ollivier-Ricci curvature, which we refer to simply as ``Ricci curvature'' throughout the remainder of the paper for brevity.} Ricci curvature \cite{OLLIVIER2007643}. Given a metric space $(X,d)$ equipped with a probability measure $m_{x}$ for each $x\in X$, for a given graph $G = (V,E)$ based on space $X$, let the edge weight of edge $xy \in E$ be $w_{xy}$, the discrete Ricci curvature of edge $xy$, $\kappa_{xy}$, is
computed as follows:
\begin{equation}
    \label{riccicurvature}
    \kappa_{xy}=1-\frac{\text{W}(m_{x}^{\alpha,p},m_{y}^{\alpha,p})}{d(x,y)}\,,
\end{equation}
where W represents the optimal mass transport distance (a.k.a. Wasserstein Distance). The mass distribution $m_{x}^{\alpha,p}$ is defined as
\begin{equation}
    \label{mass dist}
    m_{x}^{\alpha,p}=\begin{cases} 
      \alpha & i=x \\
      \frac{1-\alpha}{C}\cdot \text{exp}(-d(x,i)^{p}) & i \sim x \\
      0 & \text{Otherwise}\,,
   \end{cases}
\end{equation}
where $C = \sum_{j\sim x} \text{exp}(-(d(x, j))^{p})$, $d(x,i)$ is  shortest path taken over all edge paths from $x$ to $i$, which is the induced metric from the edge weight. The notation $i \sim x$ represents node $i$ is the neighborhood of node $x$. More specifically, given a weighted graph $(V,E,w)$ with $w_{ij}>0$ for all edges, the induced metric $d$ for edge $xy$ is defined by 
\begin{equation}
    \label{metric}
    \begin{split}
            d(x,y)=\min_{\gamma:x=v_0\leadsto v_m=y}
            \sum_{r=0}^{m-1}w_{v_{r}v_{r+1}}\,,
    \end{split}
\end{equation}
where the minimum is taken over all edge paths $\gamma$ from $x$ to $y$. It can be shown that the induced metric $d$ is well-defined. 

In addition, the Wasserstein distance in \eqref{riccicurvature} is defined as the minimum total weighted travel distance to move $ m_{x}^{\alpha,p}$ to $ m_{y}^{\alpha,p}$, i.e., $W( m_{x}^{\alpha,p}, m_{y}^{\alpha,p})=\text{inf}\{\sum_{u,v \in V}A(u,v)d(u,v)\}$ where $A$ is a discrete transport plan (map) from $V\times V$ to $[0,1]$ such that $A(u,v)$ is the amount of mass at vertex $v$ to be moved to vertex $u$ and $m_{x}^{\alpha,p}(u)=\sum_{v'\in V}A(u,v')$ and $m_{y}^{\alpha,p}(v)=\sum_{v'\in V}A(u',v)$. For efficiency, we will apply the optimal transportation distance function in python package CVXPY \cite{diamond2016cvxpy} to compute the Wasserstein distance between two mass distributions. See Appendix \ref{theoretical foundation} for more detailed explanations of theoretical concepts and optimal transport problem. Now the iteration process of discrete Ricci flow on network is defined as follows:
\begin{equation}
    \label{ricciflow}
    w_{xy}^{(i+1)}=(1-\varepsilon\kappa_{xy}^{(i)})w_{xy}^{(i)}\,,
\end{equation}
where $w_{xy}^{(i)}$ is the length of edge $xy$ at the $i$-th iteration, $0<\varepsilon\leq 1$ is the step size, and $\kappa_{xy}^{(i)}$ is the Ricci curvature in \eqref{riccicurvature} computed using the path metric $d^{(i)}$ induced by the current edge lengths $w^{(i)}$. See Algorithm \ref{algo:ricci_flow} in section \ref{Ricci-RAG}.

Note that the iteration \eqref{ricciflow} above will expand (stretch) edges with negative curvature and shrink edges with positive curvature. This process effectively causes intra-community edges (within-group connections) to condense and inter-community edges (between-group connections) to stretch. Because of this clear separation, a simple thresholding procedure can easily distinguish different communities. This is termed network ``surgery'' when edges of large weights (likely inter-community edges) are removed after several Ricci flow iterations.

\section{Ricci-Filtration}
\label{Ricci-RAG}

By viewing the group of chunked texts and user query  as a network, we propose Ricci-Filtration which further filters the initial selected chunks by their intrinsic geometric curvature with query. By definition \eqref{riccicurvature}, if two nodes $x$ and $y$ are from different communities (groups), their neighbor nodes tend to have fewer common neighbors, so the best way to move $m_{x}$ from $x's$ neighbors to $m_y$ in $y's$ neighbors has to travel along the edge $xy$ in most cases. As a result, the Wasserstein distance should be greater than the length of $xy$ ($d(x,y)$, leading to a negative Ricci curvature. Alternatively, nodes that are geometrically close or within the same community tend to share neighbors or have a shortcut between neighbors, thus the Wasserstein distance should not be greater than $d(x,y)$, leading to a positive Ricci curvature. 

In the context of RAG, after the evolution of discrete Ricci flow, we are able to see which text chunk will have positive Ricci curvature with the query in the  embedding space, which can be interpreted as intrinsic relevant chunks of query. Negative Ricci curvature indicates irrelevant chunk which will be removed before reranking. We may also get many byproducts such as the clusters of other text chunks after discrete Ricci flow. Since we only care about the relation between query and chunks, we ignore the byproducts like clustering between chunks in the current work. Exploring how to utilize the clusters formed purely from chunks is one of our future works.  Algorithm \ref{algo:ricci_flow} then shows how the discrete Ricci flow  is implemented in practice.
\begin{algorithm}
        \caption{Finite-Time Discrete Ricci Flow}
            \label{algo:ricci_flow}
\begin{algorithmic}[1]
    \REQUIRE An undirected graph $G$ formed by chunked texts. The number of flow iterations $M$. The threshold value $\eta$ for ``surgery''. The step size $0<\varepsilon\leq 1$ of discrete Ricci flow iteration. The initial edge weights are set to 1 before the first normalization.

    \FOR{$i=0,\ldots,M-1$}
    \STATE Normalize the edge weights: $w^{(i)}_{xy} \leftarrow w^{(i)}_{xy} \cdot \frac{|E|}{\sum_{uv \in E}w^{(i)}_{uv}}$, where $|E|$ is the number of edges in the graph $G$.

     \STATE Compute the path metric $d^{(i)}$ induced by $w^{(i)}$ and the Ricci curvature $\kappa_{xy}^{(i)}$ of each edge.

     \STATE Update the edge weight by $w^{(i+1)}_{xy}\leftarrow (1-\varepsilon \kappa^{(i)}_{xy})w^{(i)}_{xy}$

     \ENDFOR
     \STATE For the query node $q$, keep chunk node $j$ only if the final normalized query-edge weight satisfies $w_{qj}\leq \eta$; otherwise remove it before reranking.
    \STATE  \textbf{Return:} The filtered chunk set and the finite-time weighted graph $G$. 
\end{algorithmic}
\end{algorithm}

The loop in Algorithm \ref{algo:ricci_flow} conducts a finite number of normalized discrete Ricci-flow iterations on a weighted graph. The surgery step partitions the query-connected candidate chunks by removing high-weight query edges after the fixed iteration budget is exhausted. In the end, for edges connected with query and other chunks, we obtain the corresponding finite-time Ricci curvatures and edge weights, which can be used to indicate inter-community edges and intra-community edges. The normalization procedure rescales the edge weights so that the average edge weight is 1 at each iteration; thus Algorithm \ref{algo:ricci_flow} is a finite-time normalized Ricci flow. To solidify the theoretical foundation of Ricci-Filtration, we prove Theorem \ref{ricci-flow-detectation}, which is a normalized adaption of \citet{Ni2019CommunityDO}, to show that the normalized discrete Ricci flow  can detect community structures for certain families of graphs. 

Related well-posedness results provide useful context but do not directly prove the exact practical algorithm used here. \citet{bai2025ollivierricciflowweightedgraphs} establish existence and uniqueness for a continuous-time normalized Ollivier Ricci-flow on weighted graphs, while \citet{ma2024modifiedricciflow} introduce modified and quasi-normalized Ricci flows on arbitrary weighted graphs with global existence and uniqueness. These results support the view that Ollivier-type Ricci-flow dynamics can be formulated rigorously on weighted graphs. However, Ricci-Filtration uses an embedding-derived edge-incidence graph, forced query edges, a finite number of discrete normalized iterations, and a heuristic post-flow surgery threshold $\eta=1$. We therefore cite these works as theoretical context rather than as convergence, optimality, or correctness guarantees for Algorithm \ref{algo:ricci_flow}.

\begin{theorem}
\label{ricci-flow-detectation}
Take the complete \footnote{A graph is complete if every distinct pair of vertices is connected by exactly one unique edge and it has no loops}  graph on $b + 1$ vertices $p_1, ..., p_{b + 1}$ and
$b + 1$ complete graphs $C_1, ..., C_{b + 1}$ on $a + 1$ vertices. Take a vertex $u_i$ from each $C_i$ and identify $u_i$ with $p_i$. The
resulting graph is $G(a, b)$. Then the normalized Ricci flow associated to the Ollivier $K_{0}$-Ricci curvature detects the community structure on $G(a, b)$ if $a > b \geq 2$, namely, the weight of the intra-community edges shrink asymptotically faster than the weight of the inter-community edges, where the Ollivier Ricci curvature $K_0$ corresponds to $\alpha=0, p=0$ in equation \eqref{riccicurvature}.
\end{theorem}

\begin{proposition}[Practical-parameter separation on $G(a,b)$]
\label{prop:practical-separation}
On the same graph family $G(a,b)$ with $a>b\geq 2$, let $d_1,d_2,d_3$ denote the edge lengths for gateway-to-gateway, gateway-to-non-gateway, and non-gateway-to-non-gateway edge types, respectively. If $d_1\geq d_2\geq d_3>0$ and the node distribution in \eqref{mass dist} uses $\alpha=1/2,p=2$, then the corresponding Wasserstein updates satisfy $W_1>W_2>W_3$ and
\[
    \frac{W_3}{W_1}
    \leq
    q_{a,b}\frac{d_3}{d_1},
    \qquad
    q_{a,b}:=\frac{(a-1)(a+b)}{a(2a+b-1)}<1 .
\]
Consequently, starting from unit edge lengths and applying the unit-step normalized update, the non-gateway intra-community edge contracts geometrically relative to the inter-community edge. 
\end{proposition}

The proof of Theorem \ref{ricci-flow-detectation} utilizes the asymptotic behavior of a linear difference-equation system formed by different edge structures; see Appendix \ref{proof of ricci-flow}. Proposition \ref{prop:practical-separation}, proved in Appendix \ref{proof-alpha-half-p-two}, gives a weaker but more directly relevant analogue for the practical parameters used in our experiments. It provides idealized theoretical support for the separation mechanism observed empirically under the choice $\alpha=1/2,p=2$. It should be interpreted as a stylized explanation of why Ricci-flow filtering can separate edge types, rather than as a guarantee of empirical performance on embedding-derived retrieval graphs. Because the \((\alpha,p)=(1/2,2)\) node measures depend exponentially on the current edge lengths, this practical-parameter extension is nonlinear and does not give the same closed-form eigenvalue proof as the uniform-neighborhood case in Theorem \ref{ricci-flow-detectation}.

Together, Theorem \ref{ricci-flow-detectation} and Proposition \ref{prop:practical-separation} motivate the surgery rule by showing that, in idealized community graphs, normalized Ricci-flow updates can separate inter-community and intra-community edge types. Algorithm \ref{algo:ricci_flow} uses a finite-time version of this idea: after $M$ normalized iterations, it applies a threshold to the query edges. In the normalized implementation, the value $\eta=1$ should be read as a relative cutoff because Algorithm \ref{algo:ricci_flow} rescales the edge weights so that their average is 1 at each iteration. Thus query edges with final normalized weights below 1 are below the graph-average scale and are treated as within-community candidates, while query edges above 1 are treated as stretched edges and filtered out. This threshold worked in our experiments, but it is a heuristic finite-time surgery threshold rather than a theorem-derived universal constant.


Instead of iterating until a theoretical convergence criterion is met, our implementation uses a finite iteration budget. In practice, we found that setting $M=20$ in Algorithm \ref{algo:ricci_flow} is enough in most cases to create a useful separation of query-edge weights. How to further improve the time efficiency of Ricci-Filtration is one of our future works.

\begin{figure*}[h]
\centering
  \includegraphics[height=3.5cm,width=\linewidth]{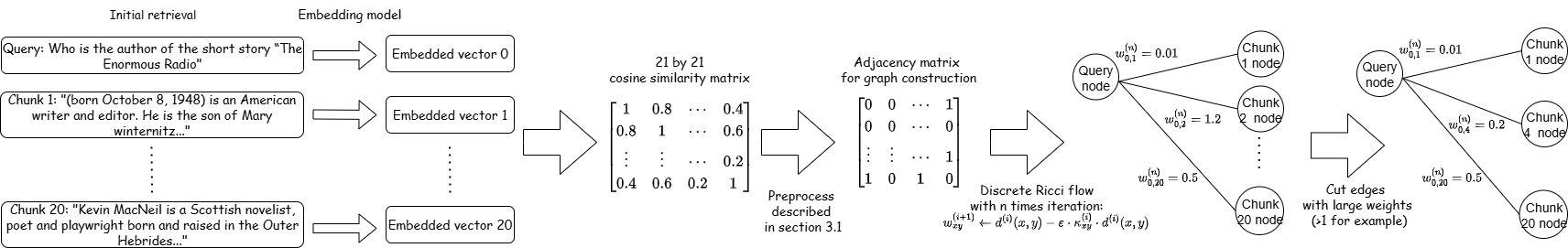}
  \caption{ Flowchart of Discrete Ricci flow for selecting chunks with positive Ricci curvature (small edge weights) with query node.  }
  \label{fig:ricciflow illustration}
\end{figure*}
\subsection{Graph construction from embeddings}
\label{graph-construction}
In Ricci-flow iteration, the neighborhood topology and the positive edge lengths play different roles. We therefore distinguish three objects: cosine dissimilarities between embeddings, a binary edge-incidence matrix defining the initial graph topology, and the positive edge lengths updated by Ricci flow. Let $z_i$ be the embedding of node $i$. For every pair of query or candidate chunk nodes, define the cosine dissimilarity
\begin{equation}
    \label{cosine-distance}
    r_{ij}=1-\operatorname{cos}(z_i,z_j)\,.
\end{equation}
Let $\tau$ be a percentile threshold computed from the off-diagonal values of $r_{ij}$ among the query and initially retrieved candidate chunks. We define the binary edge-incidence matrix $B$ by
\begin{equation}
    \label{edge-incidence}
   B_{ij}=\begin{cases}
      1 & r_{ij}\geq \tau,\ i\neq j,\\
      0 & \text{otherwise}.
   \end{cases}
\end{equation}
The entry $B_{ij}=1$ means that the edge $\{i,j\}$ is present in the computational graph; it should not be interpreted as a claim of high semantic similarity. Since the threshold is applied to cosine dissimilarity, these edges include relatively dissimilar pairs that allow the Ricci flow to stretch cross-community relations. We additionally force $B_{qj}=B_{jq}=1$ for every candidate chunk node $j$ so that every query-chunk edge is available after the flow. The initial graph and initial edge lengths are
\begin{equation}
    \label{initial-lengths}
    E_0=\{\{i,j\}:B_{ij}=1\}, \qquad w_{ij}^{(0)}=1\quad \text{for }\{i,j\}\in E_0\,.
\end{equation}
Thus the value 1 in $B$ encodes edge incidence, whereas the value 1 in $w^{(0)}$ encodes the common initial edge length. This graph-construction threshold $\tau$ is different from the post-flow surgery threshold $\eta$. Because Algorithm \ref{algo:ricci_flow} keeps the average edge weight at 1, we set $\eta=1$ as a post-normalization cutoff for deciding which query-connected chunks remain after the flow. Edges with final normalized query-edge weights less than or equal to 1 are retained as relatively close to the query, while edges above 1 are treated as stretched inter-community edges. More rigorous methods of defining both the initial graph topology and surgery threshold will be our future work.
\subsection{Illustrative example}
To see how the discrete Ricci flow iteration works, we give a simple illustrative example in this section. We randomly pick one query  from HotpotQA dataset and utilize top 20 related chunks by cosine similarity search. Other retrieval settings are the same as in section \ref{RAG}. The corresponding binary edge-incidence matrix has shape $21\times 21$. The threshold $\tau$ in \eqref{edge-incidence} was set to the 50th percentile of the cosine-dissimilarity distribution in \eqref{cosine-distance}. Flowchart \ref{fig:ricciflow illustration}  summarizes how the initial graph is constructed and how the final query-connected chunks are determined after discrete Ricci flow iteration.  We set $\alpha=0.5, p=2$ in the node distribution $m_{x}^{\alpha,p}$.

\begin{figure}[h]
\centering
  \includegraphics[width=\linewidth]{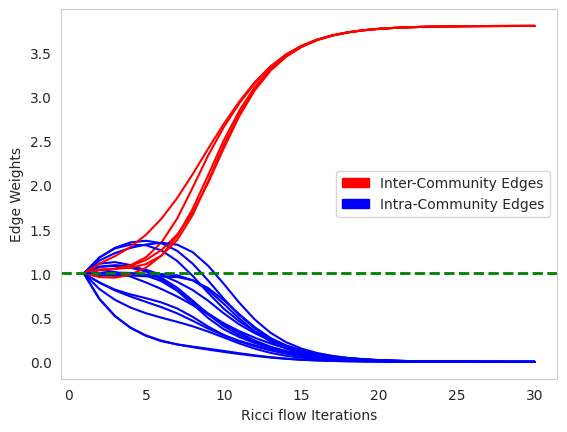}
  \caption{ Discrete Ricci flow iteration of 20 Edges connecting with user query and retrieved chunks. Inter-community edges imply the corresponding chunk nodes  are in geometrically different from query nodes.  Intra-community edges imply the corresponding chunk nodes  are in geometrically similar to query nodes.}
  \label{fig:rf_example}
\end{figure}

 Figure \ref{fig:rf_example} demonstrates the variation of edge weights  along with the finite-time Ricci-flow iteration process on graphs generated by the sampled query and chunks. When the iteration process starts, the original weight of all connected edges is set to 1. As the discrete Ricci flow evolves, we can see  a clear trend that weights of edges either move toward 0 or toward a larger positive value. In addition, edges with larger final weights will have negative Ricci curvature with query node, which are geometrically different with query and will be removed before the reranking process. On the other hand, chunks having nearly zero  edge weights, i.e., positive Ricci curvature to query node will be maintained. Therefore, the discrete Ricci flow will help us rule out irrelevant chunks dynamically and keep the both geometric and semantic associated information in the end, leading to a more robust RAG reranking system. In this figure, the surgery cutoff is $\eta=1$ after the normalization in Algorithm \ref{algo:ricci_flow}; equivalently, we keep query-connected chunks whose final normalized query-edge weight is no larger than the average edge-weight scale. This choice is a simple empirical cutoff for the normalized flow, not a claim that the unnormalized edge weights have a universal threshold at 1. As illustrated in Figure \ref{fig:rf_example}, an iteration count of 20 is sufficient to produce a visible separation of query-edge weights in this example.
 
\section{Experiments}
\label{Experiments}
\subsection{Datasets}
To comprehensively evaluate the performance of
Ricci-Filtration on general QA tasks, we follow the setting in \citet{Han2025RAGVG} and select five widely used datasets that cover different perspectives. More specifically,  we choose the Stanford Question Answering Dataset (SQuAD2.0) \cite{rajpurkar-etal-2018-know} and TriviaQA \cite{triviaqa}  for the single-hop QA task. As to the multi-hop QA task, we select HotPotQA \cite{yang-etal-2018-hotpotqa}, MultiHopRAG \cite{Tang2024MultiHopRAGBR}, and MuSiQue \cite{Trivedi2021MM} datasets. MuSiQue and TriviaQA are used in the further ablation studies.
See detailed introduction of selected datasets in Appendix \ref{Dataset}. We use Accuracy, Precision (P), Recall (R), and F1-score as evaluation metrics for the SQuAD2.0 and HotPotQA datasets, and only accuracy is reported for the TriviaQA, MultiHop-RAG and MuSiQue.

\begin{table*}[!htp]
\caption{Evaluation results (\%) on SQuAD2.0 and HotpotQA datasets with different reranker methods and generation LLMs.
}
\label{tab:hotpotandsquad}

\begin{small}
\centering
\resizebox{\linewidth}{!}{
\begin{tabular}{l|cccccccc|cccccccc}
\toprule
\textbf{Methods}      & \multicolumn{8}{c|}{\textbf{SQuAD2.0}}                                                         & \multicolumn{6}{c}{\textbf{HotpotQA}}                                                     \\ \cmidrule{2-17}
& \multicolumn{4}{c|}{\textbf{Llama 3.1-8B Instruct}}          & \multicolumn{4}{c|}{\textbf{gpt-4o-mini}} & \multicolumn{4}{c|}{\textbf{Llama 3.1-8B Instruct}}          & \multicolumn{4}{c}{\textbf{gpt-4o-mini}} \\ 
\cmidrule{2-17}

& P     & R     &F1  &\multicolumn{1}{c|}{Acc}   & P          & R         & F1    &Acc    & P     & R     & F1  & \multicolumn{1}{c|}{Acc} & P         & R         & F1     &Acc   \\
\midrule
No Reranker                       &47.64 &52.25& 48.48& \multicolumn{1}{c|}{52.00}& 65.34   & 70.89 &66.46  &69.50   & {59.43} & { 57.97} &  56.54 & \multicolumn{1}{c|}{61.50}   & {67.49}     & {67.31}  & 65.22 & 69.40  \\
Cross-Encoder Reranker                       &47.74&52.14& \textbf{48.49} & \multicolumn{1}{c|}{50.90}      &63.60   &69.46   & {64.86} & {68.20} & \textbf{62.55}& {59.57}     & {58.98}     & \multicolumn{1}{c|}{64.00}  &\textbf{69.74} &\textbf{69.84}   &\textbf{67.57} & \textbf{72.50}  \\
LLM-based Reranker  & 37.86 & 41.20 & 38.25 & \multicolumn{1}{c|}{49.50}      &37.52    &42.58    &38.56 & 52.10 & 62.45&\textbf{60.63}      & \textbf{59.49}     & \multicolumn{1}{c|}{\textbf{64.40}}    &69.46 &69.29 &67.20&71.90 \\
\textbf{Ricci-Filtration}& \textbf{52.13}&\textbf{57.28} &\textbf{53.00}&   \multicolumn{1}{c|}{\textbf{55.40}}   & \textbf{69.08} &\textbf{73.58} & \textbf{70.17} & \textbf{73.00}	& 62.39 & 59.21	&59.13	& \multicolumn{1}{c|}{63.00} &65.81 &63.97 &63.36 &67.00 \\
\bottomrule
\end{tabular}
}   
\end{small}
\end{table*}

\begin{table*}[!htp]
\caption{Accuracy (\%) on the MultiHop-RAG dataset across different query types with different reranker methods and generation LLMs.}
\label{tab:multihop}
\resizebox{\linewidth}{!}{
\begin{tabular}{l|ccccc|ccccc}
\toprule
\textbf{Methods}      & \multicolumn{5}{c|}{\textbf{LLama 3.1-8B Instruct}}                                                & \multicolumn{5}{c}{\textbf{gpt-4o-mini}}                                             \\ 
\cmidrule{2-11}
                     & \textbf{Inference}      & \textbf{Comparison}     & \textbf{Null}           & \textbf{Temporal}       & \textbf{Overall}        & \textbf{Inference}      & \textbf{Comparison}  & \textbf{Null}           & \textbf{Temporal}       & \textbf{Overall}        \\ 
\midrule
No Reranker                      &87.62 &55.96       & 57.81      &  58.83       & {66.93}          & 94.61 &63.79    &93.02          & 56.60        & {75.43}          \\

Cross-Encoder Reranker                      & 90.93 &56.66       & 60.80        &  57.46       & {68.27}          & 96.69 & \textbf{66.00}     & 94.02       & \textbf{60.21}      & \textbf{77.78}          \\
LLM-based Reranker  &90.32          & 54.09         &56.81        &  60.03          & 67.33& 95.22         & 63.67    & 93.02          & 56.60        & \textbf{75.59} \\ 
\textbf{Ricci-Filtration} & \textbf{93.87} &\textbf{58.53} & \textbf{83.39} & \textbf{60.21} & \textbf{73.12} &\textbf{97.18}	& 64.02&\textbf{95.00} & 56.60 & 76.56 \\
\bottomrule
\end{tabular}
}
\end{table*}
\subsection{Evaluation settings}
\subsubsection{RAG}
\label{RAG}
We employ cosine similarity based dense retrieval approach as our RAG
method \cite{karpukhin-etal-2020-dense}, which is common in literature. Specifically, we first split the text into chunks, each containing approximately 256 tokens. For indexing, we use OpenAI's text-embedding-3-small model, which has demonstrated effectiveness across various tasks \cite{Deng2025RevealingTN}. For the standard RAG baseline (without a reranker), we retrieve the top five chunks based on cosine similarity. To generate responses, we utilize the open-source Llama-3.1-8B-Instruct model \cite{grattafiori2024llama3herdmodels}. We also make a comparison under gpt-4o-mini to show the robustness of our method. 


\paragraph{Comparison protocol.}
All methods are evaluated from the same dense-retrieval pool and under the same downstream generation settings. For each query, the retriever first returns the top 20 candidates using the same embedding model and cosine-similarity index. The no-reranker baseline sends the top five retrieved chunks directly to the generator, while the cross-encoder and LLM-based reranker baselines rerank the same top-20 pool and then send the top five selected chunks to the generator. Ricci-Filtration is evaluated under the same retrieval pool, generation model, generation prompt, and evaluation metrics; the intended intervention is the geometric filtering step inserted before reranking. Unless otherwise stated, the remaining chunks after Ricci-Filtration are reranked with the same bge-reranker-base cross-encoder used in the cross-encoder baseline. Thus, the primary comparison against the cross-encoder baseline isolates the effect of adding Ricci-Filtration before the same downstream reranker.

\subsubsection{Reranker baseline}

For comparison, we choose cross-encoder reranker and LLM-based reranker \cite{LLMreranker} as the baseline methods. We select the top five reranked chunks from an initial pool of the top 20 similarity-based candidates. For time efficiency and cost, we employ open-source cross encoder named bge-reranker-base \cite{bge_embedding}, which is pre-trained on  large-scale pairs data using contrastive learning. In ablation studies, we tried  different cross encoders like  MS-MARCO-MINILM-L6-V2 and MS-MARCO-MINILM-L12-V2, which are fine-tuned versions of the original MiniLM model\cite{Wang2020MiniLMDS} on MS MARCO task \cite{nguyen2017ms}. As to the LLM-based reranker, we prompt gpt-4o-mini for reranking. We chose pointwise reranking because it provides clear scores that are easy to use and allows useful optimizations. The reranking prompt is given in Appendix~\ref{llm-reranker-prompt}.

\subsubsection{Ricci-Filtration}

As suggested by \citet{Ni2019CommunityDO}, we set $\alpha=0.5, p=2$ in the node distribution $m_{x}^{\alpha,p}$. For large $p$, the far neighbors of the source node will be heavily discounted. Wasserstein distance in \eqref{riccicurvature} is calculated by optimal transportation distance function in python package CVXPY \cite{diamond2016cvxpy} . We compute cosine dissimilarities from the embedded vectors and use them to construct the binary edge-incidence matrix in \eqref{edge-incidence}. A threshold is required to convert these pairwise dissimilarities into a graph topology. In the following experiments, we set $\tau$ to the 50th percentile of the off-diagonal cosine-dissimilarity values\footnote{The ablation tables report this graph-construction threshold as a percentile. For example, 50\% means that $\tau$ is the median off-diagonal cosine dissimilarity, and a non-query pair is included when $r_{ij}\geq \tau$.}. In addition, we force the query node to be connected with each chunk node in order to have all information of edges corresponding to query and chunks such as weights available after discrete Ricci flow iteration. Each present edge is initialized with unit length before the finite-time normalized flow. The finite-time iteration budget is set to $M=20$ based on our experience. To keep the comparison with the cross-encoder baseline controlled, we use the same bge-reranker-base cross-encoder after Ricci-Filtration; the difference is that the cross-encoder baseline reranks all 20 retrieved candidates, whereas Ricci-Filtration first removes geometrically irrelevant chunks and then reranks the remaining subset.

\subsection{Results}

\begin{table*}[h]
  \caption{Ablation studies of Ricci-Filtration on the Accuracy (\%) of TriviaQA dataset}
  \label{ablation-study}
  \begin{center}
  \begin{small}
      \begin{sc}
        \begin{tabular}{l|c|c|c|c|c|c|c|c}
          \toprule
          Methods &$k$       & $n^{*}$  &$\tau$ &$p$ & $M$ &Embedding model&Reranker model&Acc \\
          \midrule
          Cross-Encoder     &20  & 20&-& -& -& text-embedding-3-small &  ms-marco-MiniLM-L12-v2 &69.00\\
          
          \hline
        Cross-Encoder     &20  & 10&-& -& -& text-embedding-3-small & ms-marco-MiniLM-L12-v2&69.00\\
          
          \hline
                    
            Cross-Encoder     &20  & 5&-& -& -& text-embedding-3-small &  ms-marco-MiniLM-L6-v2 &68.00\\
                      \hline
                  Cross-Encoder     &20  & 5&-& -& -& text-embedding-3-small &  ms-marco-MiniLM-L12-v2 &69.00\\
          
          \hline
          Ricci-Filtration    & 10& 7.0&50\%& 2& 20& text-embedding-3-small &  ms-marco-MiniLM-L12-v2 &\textbf{72.00}\\
                    \hline
          Ricci-Filtration    & 20 & 16.7&50\%& 2& 20& text-embedding-3-small & ms-marco-MiniLM-L12-v2 &\textbf{74.00}\\ 
          \hline
          Ricci-Filtration    & 20 & 15.9& 50\%& 5& 20& text-embedding-3-small & ms-marco-MiniLM-L12-v2&70.00\\
                    \hline
          Ricci-Filtration    & 20 & 13.5& 50\%& 1& 20& text-embedding-3-small & ms-marco-MiniLM-L12-v2 &69.00\\
          \hline
            Ricci-Filtration   & 20 & 17.0& 25\%& 2& 20& text-embedding-3-small &ms-marco-MiniLM-L12-v2 &69.00\\
          \hline
          Ricci-Filtration    & 20 & 16.0& 75\%& 2& 20& text-embedding-3-small & ms-marco-MiniLM-L12-v2 &64.00\\        
                    \hline
            Ricci-Filtration   & 20 & 14.8& 50\%& 2& 10& text-embedding-3-small & ms-marco-MiniLM-L12-v2 &\textbf{71.00}\\
          \hline
          Ricci-Filtration    & 20 & 10.4& 50\%& 2& 1& text-embedding-3-small & ms-marco-MiniLM-L12-v2 &65.00\\  
          \hline
                      Ricci-Filtration   & 20 & 16.6& 50\%& 2& 20& text-embedding-ada-002 & ms-marco-MiniLM-L12-v2 &66.00\\
                                \hline
          Ricci-Filtration    & 20 &16.7&50\%& 2& 20& text-embedding-3-small & bge-reranker-base&\textbf{72.00}\\ 
          \hline
          Ricci-Filtration    & 20 &16.7&50\%& 2& 20& text-embedding-3-small & ms-marco-MiniLM-L6-v2 &\textbf{74.00}\\ 
                    \hline
        \end{tabular}
      \end{sc}
      \end{small}
        {\scriptsize\emph{Note.} $n$ for Ricci-Filtration methods represents the average number of chunks fed into reranker model after Ricci-Filtration.\par}
  \end{center}
  
\end{table*}
Table \ref{tab:hotpotandsquad} demonstrates the comparisons of different reranker methods on SQuAD2.0 and HotpotQA datasets where precision, recall, F1-score and accuracy are reported. We also tried different generation LLMs to see the robustness of different baselines.  Following the standard practice in \citet{Han2025RAGVG}, we only report accuracy for MultiHop-RAG in Table \ref{tab:multihop}. The best performing metrics under different generation LLMs are highlighted in bold. From the results in Table \ref{tab:hotpotandsquad} and Table \ref{tab:multihop} , we can see that Ricci-Filtration has superior performance in SQuADv2 and MultiHop-RAG datasets. In SQuADv2, Ricci-Filtration dominates all the metrics in under both generation LLMs and improves the metrics by about 5\% in general. Similar performance can be observed in MultiHop-RAG with a significant improvement in the accuracy in null queries (from 60.80\% to 83.39\%), showing the potential of Ricci-Filtration for specific query type. 

Notably, we observe that Ricci-Filtration has weakest performance on the HotpotQA benchmark across both generation LLMs, suggesting its limitation in multi-hop reasoning tasks where only critical and connected chunks should be used. However, current Ricci-Filtration tends to maintain excessive number of chunks after iteration, introducing noise that misleads the reranker during multi-step reasoning. Ablation studies on MuSiQue (Appendix \ref{MuSiQue}) also confirm this trend.  There are clues that the margin of Ricci-Filtration might be diminished when strong generation LLMs are used, which is expected. However, the improvement brought by Ricci-Filtration suggests that geometric filtering can improve answer quality without using an LLM to make the filtering decision. This should be interpreted as a resource trade-off rather than a blanket deployment-cost claim, since the Ricci-flow iterations introduce measurable latency (Appendix \ref{time efficiency tab}).

As to the LLM-based reranker, we see that its performance varies among different datasets. One reason is that the ability of LLM-based rerankers relies heavily on the underlying reranking LLM and reranking prompt. It is also sensitive to workflow design. Another drawback is the per-query API cost of scoring retrieved chunks. To keep the experiment bounded, we use gpt-4o-mini as the reranker model. We found that replacing it with stronger models can improve results, but at a substantially higher API cost. Compared with this LLM-based reranker, Ricci-Filtration avoids LLM calls during the filtering step, but its iterative graph computation still incurs latency, so the total deployment cost depends on hardware, implementation, and workload.


\subsection{Ablation studies}
\label{ablation}

To evaluate the individual contributions of its components, we conduct ablation studies on Ricci-Filtration using Llama-3.1-8B-Instruct. Table \ref{ablation-study} details performance variations across node distribution parameters, graph-construction threshold percentiles, iteration times, and model selections (embeddings and rerankers). We utilize cross-encoder based reranker as the primary baseline due to its competitive performance in previous evaluations. Additional clustering-baseline results are provided in Appendix~\ref{kmeans-ablation}.

In summary, Ricci-Filtration generally outperforms the simple cross-encoder-based reranker under various comparable settings (see bold results in Table \ref{ablation-study}), highlighting the framework's robustness across different architectures. On the other hand, Ricci-Filtration exhibits relatively low sensitivity to the choice of the node distribution parameter $p$ and the underlying reranker model. Regarding the graph-construction threshold percentile, a moderate value yields optimal performance. While a lower threshold creates a dense graph that may introduce excessive noise, an overly high threshold discards critical structural information required for the Ricci flow iterations. The number of iterations also significantly impacts performance and a sufficient number of iterations is essential for the Ricci-Filtration process to achieve decent results. Finally, the performance of  Ricci-Filtration  relies on the underlying embedding model, which is expected because the initial graph is constructed by cosine-dissimilarity thresholding over embedding vectors.  If the  underlying embedding space is not precise, the results of discrete Ricci flow iteration are likely to be misleading. Furthermore, the last two rows of Table \ref{ablation-study} demonstrate that Ricci-Filtration maintains its superior performance under different reranker models, highlighting the framework's robustness across different architectures.

\section{Discussion and Limitations}
\label{discussion}

This work focuses on Ricci-Filtration for QA tasks. However, applying it to sensemaking queries (query-focused summarization) is more complex \cite{GraphRAG}, as these require analyzing global trajectories and entity connections. While performing discrete Ricci flow across the entire corpus could lead to GraphRAG-style community detection \cite{GraphRAG}, the computational cost of iterative curvature evaluation is currently prohibitive. Future research will focus on optimizing Ricci flow efficiency to support global sensemaking tasks on large-scale graphs. 

 Currently, we use cosine-dissimilarity thresholding to construct the initial edge-incidence graph for its simplicity. However, as shown in Table \ref{ablation-study}, Ricci-Filtration is highly sensitive to the quality of the underlying embedding model. Exploring more robust methods for initial graph construction remains a key research direction.
On the other hand, experiments for HotpotQA and ablation study for MuSiQue (Table \ref{ablation-study-musique} in Appendix \ref{MuSiQue}) imply that current Ricci-Filtration has limitations in multi-hop reasoning where only critical and connected chunks should be used. How to combine the  ``geometric denoising'' ability of discrete Ricci flow with the reasoning ability of RAG reranker appropriately is an important research direction. 

The major limitation of Ricci-Filtration is its time complexity compared to other methods, as shown by Table \ref{time-table} in Appendix \ref{time efficiency tab}. While 20 iterations were used for most experiments, ablation studies suggest that 10 iterations remain effective. Consequently, developing a more time efficient iteration method with optimized weight cut-offs is also a critical direction for future application.

\section*{Impact Statement}
This paper presents a method for enhancing the performance of RAG with reranker through a differential geometry method.  The broader impact of this work involves contributing to the development of more reliable and factually grounded AI assistants. This has direct applications in critical domains such as legal research, medical information retrieval, and education, where reducing ``hallucinations'' and ensuring source-verifiability are paramount for user safety and trust. On the other hand,  while our method improves accuracy, it remains susceptible to biases inherent in the underlying retrieval corpus. If the source data contains historical or social prejudices, the Ricci-Filtration may still prioritize these biased viewpoints, thereby amplifying them in the final generation. Furthermore, the deployment of multi-stage RAG pipelines increases the total inference-time computational overhead. We encourage practitioners to consider the environmental impact of increased energy consumption and to implement our system alongside robust bias-detection audits and data-governance frameworks to ensure equitable and sustainable deployment.


\bibliography{custom}
\bibliographystyle{icml2026}

\newpage
\appendix
\onecolumn

\section*{Appendix}
\section{Question Answering Dataset}
\label{Dataset}

For QA tasks, we use the following four widely
used datasets:
\begin{enumerate}
    \item \textbf{SQuAD2.0 }\cite{rajpurkar-etal-2018-know}:  The Stanford Question Answering Dataset (SQuAD) is a reading comprehension benchmark composed of questions created by crowdworkers based on a collection of Wikipedia articles. For each question, the answer is either a specific text span drawn from the associated passage or, in some cases, the passage does not contain an answer. SQuAD2.0 extends SQuAD1.1\cite{squad1} by retaining its original answerable questions and adding many adversarially written unanswerable questions that are designed to closely resemble answerable ones. To ensure a more challenging evaluation, we randomly selected 1,000 questions from the validation set of SQuAD2.0, which also involves unanswerable questions. Additionally, we treat SQuAD2.0 as a multi-document QA task and build a single RAG system to handle all questions. If Ricci-Filtration performs well here, it proves  generalizability.

    \item \textbf{TriviaQA} \cite{triviaqa}: A large-scale dataset with 650k+ question-answer-evidence triples. It is known for having more complex, ``fact-heavy'' questions that require a high-precision reranker to find the specific snippet of knowledge. Since the typical context is long for each question,  we  randomly selected 100 questions from the validation set of TriviaQA for evaluation and we build a uniform RAG system to handle all the selected questions.
    \item \textbf{Hotpot}\cite{yang-etal-2018-hotpotqa}  We utilize HotpotQA, a benchmark dataset for multi-hop reasoning that associates 10 context paragraphs with each question. Since modern LLMs can readily solve the dataset's simpler queries, we focus our evaluation on a more rigorous subset comprising 1,000 'hard bridging' questions randomly sampled from the validation set. We treat HotpotQA as a multi-document QA task, employing a unified RAG system to process all queries. 
    \item \textbf{MultiHop-RAG} \cite{Tang2024MultiHopRAGBR} MultiHop-RAG is a QA dataset designed to evaluate retrieval and reasoning across multiple documents with metadata in RAG pipelines, where document metadata aims to reflect complex scenarios commonly found in real-world RAG applications. Constructed from English news articles, it contains 2,556 queries with supporting evidence distributed across 2 to 4 documents. The dataset includes four query types: Inference queries synthesize claims about a bridge entity to identify it; Comparison queries compare similarities or differences and yield ``yes'' or ``no'' answers in general; Temporal queries examine event ordering with answers like ``before'' or ``after''; and Null queries where no answer can be derived from the retrieved documents. For evaluation, we view it as a single-document QA task. 
\item \textbf{MuSiQue} \cite{Trivedi2021MM} MuSiQue (Multihop Single-hop Question Composition) is a challenging dataset for multi-hop Question Answering (QA) designed to be hard to cheat. It enforces connected reasoning, where one reasoning step critically relies on information from another. It was built bottom-up by first taking existing single-hop questions from datasets like SQuAD and Natural Questions, then systematically linking these questions together. Rigorous filters were employed to ensure the resulting multi-hop question could not be answered if any intermediate step was skipped. The dataset involves approximately 25,000 examples with 2-4 hop questions. Similarly, we randomly selected 1,000 questions from the validation set of MuSiQue, treated it as a single document QA task and built a single RAG system to handle all questions for the challenging evaluation.
\end{enumerate}

\section{Theoretical foundation}
\label{theoretical foundation}

Note that most notations and definitions in this section are following the work in \citet{Ni2019CommunityDO}. For more detailed introductions, interested readers can refer to the original paper of \citet{Ni2019CommunityDO}.
\subsection{Notations and definitions}

We represent a network as an unweighted graph $G = (V, E)$, defined by a vertex set $V$ and an edge set $E$. The objective of community detection is to partition $G$ into a collection of $n$ disjoint, connected subgraphs $\{C_1, C_2, \dots, C_n\}$, referred to as communities. This is achieved by identifying a set of inter-community edges, those connecting distinct clusters, the removal of which isolates the subgraphs. Formally, a robust community structure is characterized by two fundamental properties
\begin{enumerate}
    \item High Intra-community Density: Vertices within any given community $C_i$ exhibit a high degree of connectivity.
    \item Low Inter-community Density: Vertices belonging to different communities are sparsely connected.
\end{enumerate}

Such topological features are common in empirical networks, including social, biological, and technological systems

We will use the following definitions and conventions in \cite{Ni2019CommunityDO}. Two vertices $i, j \in V$ are said to be adjacent ($i \sim j$) if there exists an edge $ij \in E$. We extend the unweighted model to a weighted graph (or metric graph) by defining a weight function $w: E \to \mathbb{R}_{\geq 0}$, which assigns a non-negative scalar $w_{ij}$ to each edge. A path of length $n$ between nodes $a$ and $b$ is defined as a sequence of edges $\{e_0, e_1, \dots, e_{n-1}\}$ where $e_i = v_i v_{i+1}$, such that $v_0 = a$ and $v_n = b$. The
length of the path $\{e_{0},....,e_{n-1}\}$ is defined to be $\sum_{i=0}^{n-1}w_{i(i+1)}$. The path is said to have n hops.

\subsection{The optimal transport problem}

The classical formulation of the optimal transport problem, originally proposed by Gaspard Monge in 1781, concerns the identification of a transport map that minimizes the total cost associated with relocating mass from a source distribution to a target distribution. In its physical motivation, this corresponds to the efficient transfer of raw materials (e.g., iron ore) from production sites to industrial sinks. To begin, let us briefly recall the notion of metric spaces and Borel measures on a metric
space. We define a metric space as a pair $(X;d)$ where $X$ is a set and $d$ is a distance function $d : X \times X \to \mathbb{R}_{>0}$ with the following
properties:
\begin{itemize}
    \item $d(x,y)=0$ if and only if $x=y$;
    \item $d(x,y)=d(y,x)$;
    \item $d(x,y)+d(y,z)\geq d(x,z)$ for all $x,y,z\in X$.
\end{itemize}

On a given metric space $(X, d)$, we consider the Borel $\sigma$-algebra, generated by the collection of all open sets in $X$ under countable unions, countable intersections, and relative complements. A Borel probability measure $\mu$ is a mapping from the Borel sets to the interval $[0, 1]$ that satisfies:

\begin{itemize}
    \item 
Normalization: $\mu(X) = 1$
\item Countable Additivity: For any sequence of pairwise disjoint Borel sets $\{A_i\}_{i=1}^{\infty}$,$$\mu\left(\bigcup_{i=1}^{\infty} A_i\right) = \sum_{i=1}^{\infty} \mu(A_i)$$
\end{itemize}

In the context of a finite graph $G = (V;E)$, the metric space $(X,d)$ is specialized by setting $X$ to be the discrete vertex set $V$. Under the discrete topology, the Borel $\sigma$-algebra is equivalent to the power set $\mathcal{P}(V)$; thus, every subset of $V$ is a Borel set. A Borel probability measure $\mu$ on $V$ is consequently characterized by a probability mass function $\mu:V\to [0,1]$ that satisfies the normalization condition  $\sum_{i\in V}\mu(i)=1$.

Given an edge weighted graph $(V;E;w)$ with $w_{ij} > 0$ for all edges, one introduces a metric $d$ on vertex set $V$ by

\begin{equation}
    \label{path-metric-appendix}
    d(v,v') = \min_{\gamma:v\leadsto v'} \sum_{e\in \gamma} w_e .
\end{equation}
where the minimum is taken over all edge paths from $v$ to $v'$. We call $d$ the induced metric from the edge weight $w$.
Note that by definition, $d(v,v')+d(v',v'')\geq d(v,v'')$ for any three vertices $v,v',v'' \in V$. This is exactly the definition we give in section \ref{discrete ricci flow}.

Consider two metric spaces, $X$ and $Y$, endowed with Borel probability measures $\mu$ and $\nu$, respectively. In the physical context of the Monge problem, $X$ and $Y$ represent the spatial distributions of supply centers (e.g., mines) and demand centers (e.g., factories), while $\mu$ and $\nu$ characterize the respective densities of material to be transported and consumed. We define a continuous cost function $c: X \times Y \to \mathbb{R}_{>0}$, where $c(x, y)$ denotes the cost incurred by transporting a unit of mass from location $x \in X$ to $y \in Y$. In the specific case where $X = Y$ and the transportation cost is proportional to the displacement under a constant rate, the cost function is typically identified with the metric of the underlying space, such that $c(x, y) = d(x, y)$

A transport map $T: (X,\mu) \to (Y,\nu)$ is a measure preserving map, i.e., for any Borel set $A\subset Y$, $\nu(A)=\mu(T^{-1}(A))$.
Monge's formulation of the optimal transportation problem is to find a transport map $T:X\to Y$ that realizes the infimum
\[
\inf \bigg\{\int_{X}c(x,T(x))d\mu(x) \bigg| \text{T is a transportation}\bigg\}
\]

A transformation $T: X \to Y$ that achieves the aforementioned infimum is termed an optimal transport map. In this general setting, the existence of such a map is not guaranteed by classical results, as the Monge formulation is highly non-linear and subject to strict feasibility constraints. A significant advancement in the field was established by \citet{Kantorovitch1958OnTT}, who reformulated the problem into a linear optimization framework, ensuring the existence of a solution. Kantorovich relaxed the deterministic map $T$ into a transportation plan $\gamma$, defined as a Borel probability measure on the product space $X \times Y$. To be considered a valid coupling, $\gamma$ must satisfy  $\gamma(A\times Y)=\mu(A)$  and $\gamma(X\times B)=\mu(B)$ for all Borel sets $A$ and $B$. The goal is to find a transportation
plan $\gamma$ that attains the infimum cost

\[
W(\mu,\nu)=\inf \bigg\{\int_{X\times Y}c(x,y)d\gamma(x,y) \bigg| \gamma \in \Gamma(\mu,\nu)\bigg\}
\]

where $\Gamma(\mu,\nu)$ denotes the collection of all transportation plans. If $X=Y$, the quantity $W(\mu,\nu)$ is called the Wasserstein
distance between two probability measures $\mu,\nu$ on $X$. Kantorovich proved that the infimum is always achieved by
some transportation plan.

In the context of a finite weighted graph $G = (V, E, w)$, the Kantorovich problem admits a discrete reformulation. Let $d: V \times V \to \mathbb{R}_{\geq 0}$ denote the path metric induced by the edge weights, as defined in Equation (1). A transportation plan $\gamma$ is given by a map $\gamma : V\times V \to [0,1]$ such that $\sum_{i\in V}\gamma_{ij}=\mu_{j}$ and$\sum_{j\in V}\gamma_{ij}=\nu_{i}$for all $i,j \in V$.
The goal is to find the minimum cost
\[
min\bigg\{\sum_{i,j\in V}\gamma_{i,j}d(i,j): \gamma \text{ is a transportation plan}\bigg\}
\]

which is a linear programming problem and solution is computationally tractable using standard optimization algorithms, such as the simplex method or interior-point methods, making Wasserstein distance a robust metric for analyzing distances between distributions on graphs.

\subsection{Curvatures in classical differential geometry}
A foundational pillar of contemporary geometry is the concept of curvature, a quantitative measure of spatial deviation from flatness. Formally conceptualized by Gauss and Riemann in the 19th century, the study of curvature begins with the $n$-dimensional manifold: a topological space that is locally homeomorphic to Euclidean $n$-space. When such a manifold is equipped with a Riemannian metric, which assigns an inner product to each tangent space, it becomes a Riemannian manifold, the primary object of inquiry in the field.

Historically, these concepts emerged from the study of smooth surfaces ($n=2$) embedded in $\mathbb{R}^3$. For a surface $S$, the Gauss map $\mathcal{G}: S \to S^2$ maps each point $p \in S$ to its unit normal vector. The Gaussian curvature $K$ at $p$ is defined as the Jacobian determinant of this map, representing the signed area distortion.

Under this framework, geometric topographies are categorized by their curvature: Euclidean planes exhibit zero curvature ($K = 0$), spheres exhibit constant positive curvature ($K > 0$), and hyperboloids of one sheet exhibit negative curvature ($K < 0$). A pivotal advancement was Gauss's Theorema Egregium, which demonstrated that Gaussian curvature is an intrinsic property. It depends solely on the Riemannian metric and remains invariant regardless of how the surface is isometrically embedded in higher-dimensional space.

Riemann extended these principles to higher dimensions through the introduction of sectional curvature. For a Riemannian manifold $(M, g)$, the sectional curvature $K(P)$ assigns a scalar value to each 2-dimensional plane $P$ within the tangent space $T_pM$. This scalar corresponds to the Gaussian curvature of the surface formed by the image of $P$ under the exponential map.

The sign of the curvature dictates the global behavior and ``density'' of the manifold. Positively curved spaces tend toward ``crowded'' geometries with small diameters, whereas negatively curved spaces exhibit a "spreading" geometry, often characterized by infinite fundamental groups, contractible universal covers, and large-scale branching similar to a tree structure.

The Ricci curvature is a fundamental tensor that assigns a scalar value to each unit tangent vector $v$ at a point $p$, representing the average of the sectional curvatures of all 2-dimensional planes containing $v$. In modern geometry, this curvature is often characterized by its influence on the metric properties of the space. Specifically, Ricci curvature governs the volume growth rate of geodesic balls relative to their radii. Furthermore, it determines the volume of intersection between two overlapping balls as a function of their radii and the displacement between their respective centers.

There have been various approaches to generalize the notion of curvature to spaces which are not manifolds (e.g., a
graph with edge weights). One of the important work of \citet{OLLIVIER2007643} which relates Ricci curvature to optimal transport. Since optimal transport can be formulated on very general metric spaces with probability measure at each point, in particular
on networks with edge weights and probability measures at each vertex, it facilitates the application in community
detection using curvature and optimal transport \cite{Ni2019CommunityDO}.

\subsection{The Ricci flow}
\label{the ricci flow}
Introduced by \citet{ricciflow}, the Ricci flow is a fundamental tool in geometric analysis that deforms the metric of a Riemannian manifold through a process formally analogous to the diffusion of heat. By evolving the metric in a way that smooths out geometric irregularities, the flow serves as a non-linear counterpart to the heat equation. Over the past forty years, it has emerged as one of the most powerful techniques for resolving profound geometric problems, most notably providing the framework for the proof of the Poincare conjecture.

Given a Riemannian manifold $M$ equipped with a metric $g_{ij}$ and its corresponding Ricci curvature tensor $R_{ij}$, Hamilton's Ricci flow is defined as a second-order nonlinear partial differential equation governing the evolution of symmetric $(0, 2)$-tensors:
$$\frac{\partial}{\partial t} g_{ij}(t) = -2 R_{ij}(t)$$

A solution to the Ricci flow is a one-parameter family of metrics $g_{ij}(t)$ on a smooth manifold $M$ satisfying the above partial differential equation. One of the key properties of the Ricci flow is that the curvature evolves according to a nonlinear version of the heat equation. Thus the Ricci flow tends to smooth out irregularity of
the curvature. Under the Ricci flow, regions in the manifold of positive sectional curvature tend to shrink and regions of negative sectional curvature tend to expand and spread out. Singularities usually occur while deforming a Riemannian 3-manifold through the Ricci flow. They appear in a small neighborhood of a surface in the
3-manifold. By removing the singularities (i.e., surfaces) and redefining the Ricci flow on the remaining pieces, one produces the Ricci flow with surgery on the manifold. The ground-breaking work of \citet{perelman2002entropyformularicciflow} shows that the Ricci flow with surgery captures the geometric decomposition of the 3-manifold. It solves the Geometrization Conjecture of Thurston and geometrically classifies all 3-manifolds.

Ricci flow provides a robust framework for analyzing network evolution and community detection. Under this heuristic, a network is treated as a discrete approximation of a high-dimensional manifold (such as a 3-manifold), where distinct network communities correspond to the components of a geometric decomposition. Given that \citet{perelman2002entropyformularicciflow} demonstrated the capacity of Ricci flow to identify the geometric components of 3-manifolds, it follows that a discrete Ricci flow can be leveraged to uncover community structures within networks. Analogous to the Hamilton-Perelman framework, the specific iteration cutoffs and surgery thresholds for the flow must be calibrated to the unique topological properties of each individual network.
\section{Proof of Theorem \ref{ricci-flow-detectation}}
\label{proof of ricci-flow}
We start by computing the Wasserstein distance of a metric on $G(a;b)$. Note that in each community$ C_i$ there is a specific
node $u_i$ which connects to other communities. According to the proof in \cite{Ni2019CommunityDO}, we call this node the gateway node and the rest of the nodes in $C_i$ the
non-gateway nodes. As shown in figure \ref{fig:community type}, there are three types of edges in the graph:

\begin{enumerate}
    \item Edges connecting two communities (on two gateway
nodes, such as $uv$ in figure \ref{fig:community type};
\item Edges connecting a gateway node with a non-gateway node in the same community, such as $ui,uj$ figure \ref{fig:community type};
\item  And edges connecting two non-gateway nodes (such as $ij$).
\end{enumerate}
\begin{figure}[h]
\centering
  \includegraphics[width=0.5\linewidth]{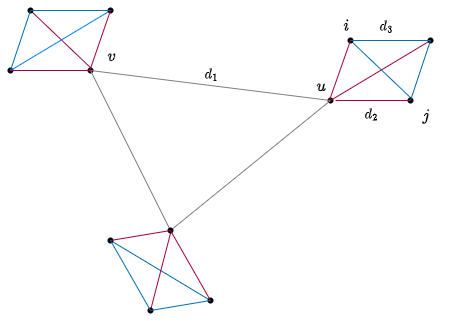}
  \caption{ An example graph $G(a;b)$ obtained from a complete graph on $b+1$ vertices by replacing each vertex by a complete graph of $a+1$
vertices. In this figure, $a = 3, b = 2$. This plot is essentially the same as the one given by \citet{Ni2019CommunityDO} }
  \label{fig:community type}
\end{figure}
\begin{figure}[htp]
\centering
\begin{minipage}{0.32\linewidth}
\centering
\includegraphics[width=\linewidth]{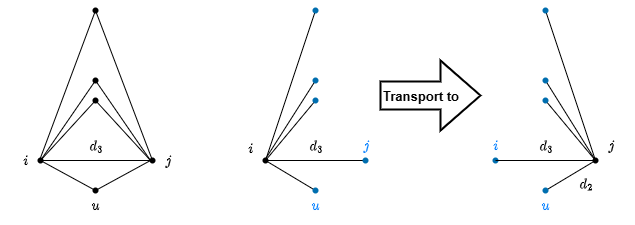}\\[-0.5ex]
\small (a)
\end{minipage}%
\hfill
\begin{minipage}{0.32\linewidth}
\centering
\includegraphics[width=\linewidth]{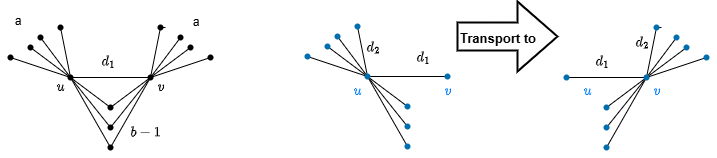}\\[-0.5ex]
\small (b)
\end{minipage}%
\hfill
\begin{minipage}{0.32\linewidth}
\centering
\includegraphics[width=\linewidth]{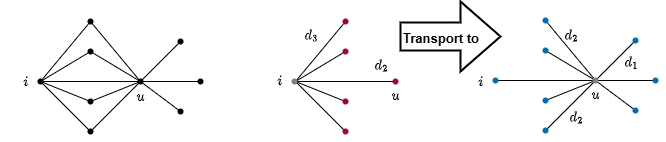}\\[-0.5ex]
\small (c)
\end{minipage}%

\caption{Parts (a), (b) and (c) from \cite{Ni2019CommunityDO} illustrate the optimal transportation to move the mass at vertex $u$ to vertex $v$ under different types of edges. }
\label{types of optimal transport}
\end{figure}

In Section \ref{graph-construction}, the initial metric assigns unit length to each edge of the constructed graph. In addition, iteration \eqref{ricciflow} implies that the Ricci flow preserves the graph symmetry. Note that  there are only three
different edge lengths at each iteration of the Ricci flow, corresponding to the three types of edges.  In addition, we assume the  edge
lengths of the $(n+1)$-th iteration be $W_1$, $W_2$ and $W_3$ which are the Wasserstein distances of the corresponding edges for
the metric graph$ (G(a;b);d)$ with respect to the probability measures $\{\mu_{x}|x\in V\}$.

Suppose the edge lengths of the metric at the $n$-th iteration are $d_{n,1}$, $d_{n,2}$ and $d_{n,3}$ for the edges between communities, edges from a gateway
node to a non-gateway node, and edges between two non-gateway nodes respectively, as shown in figure \ref{fig:community type}. Then we have  $d_{n,1}=w_{uv}=W_{n,1},d_{2}=w_{uj}=W_{n,2},d_{3}=w_{ij}=W_{n,3}$ since there is only one path connecting each pair. This observation is critical in the following derivation. Lemma \ref{lemma} facilitates the asymptotic analysis of edge lengths under normalized discrete Ricci flow in Algorithm \ref{algo:ricci_flow}. For light notation, we ignore the number of iteration subscript ($n,n+1$) for Wasserstein distance and edge length in the proof of Lemma \ref{lemma}. 

\begin{lemma} 
\label{lemma}
According to the notations above and assumptions in Theorem statement, the Wasserstein distances $W_1,W_2,W_3$ are given by 
\begin{itemize}
    \item $W_1=\frac{a-1}{a+b}d_1+\frac{2a}{a+b}d_{2}$
    \item $W_{2}=\frac{b}{a+b}d_{1}+\frac{ab-a-b}{a(a+b)}d_{2}$
    \item $W_{3}=\frac{1}{a}d_{3}$
\end{itemize}

In addition, assume that $a>b\geq 2$ and $d_1\geq d_2\geq d_3$. Then we have $W_1\geq W_2 \geq W_3$.
\end{lemma}

\begin{proof}
    We derive the identity of $W_3$ first as it is the most straightforward case shown in Figure \ref{types of optimal transport}(a). Note that nodes $i$ and $j$ share the same vertical vertices. Thus we only need to move mass at node $i$ ($1/a$) to node $j$ along $ij$ to finish the transport, which is optimal by definition. It follows that $W_{3}=\frac{1}{a}d_{3}$.

    Next, we consider the identity of $W_{1}$, which moves the probability measure $\mu_{u}$ to $\mu_{v}$.  By definition and Figure \ref{types of optimal transport}(b), the degrees of nodes $u$ and $v$ are all equal to $a+b$. By the node distribution parameter, both probability measures $\mu_{u}$ and $\mu_{v}$ have density(mass) $\frac{1}{a+b}$ at each vertices connected to $u$ and $v$. Note that $u$ and $v$ share the middle $b-1$ vertices where each of them has mass $\frac{1}{a+b}$ under both measures $\mu_{u}$ and $\mu_{v}$.  For the sake of minimal effort in transportation from $\mu_{u}$ to $\mu_{v}$, there is no need to move them. For the rest of $a$ many vertices $x$ adjacent to $u$ and $a$ many vertices $y$  adjacent to $v$, we consider the following transportation plan, which is the best one:

    \begin{enumerate}
        \item For each vertex $x$ adjacent to $u$, where the mass at $x$ is $\frac{1}{a+b}$, we move it along edge $xu$ from $x$ to $u$. The total cost of moving all of them is $\frac{a}{a+b}d_{2}$.
        \item Leave a mass of $\frac{1}{a+b}$ at $u$ and move the rest of mass to $v$, where the total cost would be $\frac{a-1}{a+b}d_1$.
        \item After the previous two steps, the total mass at vertex $v$ is $\frac{a}{a+b}=\frac{a-1}{a+b}+\frac{1}{a+b}$. We now evenly distribute those mass to the $a$ many vertices $y$ adjacent to $v$, where the total cost would be $\frac{a}{a+b}d_{2}$
    \end{enumerate}

    Therefore, $W_1=\frac{a-1}{a+b}d_1+\frac{2a}{a+b}d_{2}$.

As to the equation for $W_{2}$, notice that the mass of $\mu_{i}$ at a vertex $x$ adjacent to $i$ is $\frac{1}{a}$ and the mass of $\mu_{u}$ at a vertices  adjacent to $u$ is $\frac{1}{a+b}$. In addition, every vertex adjacent to $i$ is also adjacent to $u$. The optimal transportation is as follows:

\begin{enumerate}
    \item Leave the mass $\frac{1}{a+b}$ at each vertex $x\neq u$ adjacent to $i$. Then the total leftover mass at those vertices is $(a-1)(\frac{1}{a}-\frac{1}{a+b})=\frac{(a-1)b}{a(a+b)}$. Move the total mass of $\frac{1}{a+b}$ from the leftover mass at these $x$ to the vertex $i$ of distance $d_3$ from $x$. Note that there are many ways to achieve this. For example, for each of these $(a-1)$ vertices, we move $\frac{1}{(a+b)(a-1)}$ from each of them to $i$, then the total cost would be $\frac{d_3}{a+b}$. This step finishes the mass transportation for $i$.

    \item After  step 1, there is a total mass of $\frac{(a-1)b}{a(a+b)}-\frac{1}{a+b}=\frac{ab-a-b}{a(a+b)}$ at these x. Similarly, we can move the mass to vertex $u$ along edges of length $d_{2}$, generating total cost $\frac{ab-a-b}{a(a+b)}d_{2}$.
    \item Now the mass at vertex $u$ becomes $\frac{1}{a}+\frac{ab-a-b}{a(a+b)}=\frac{b}{a+b}$. We then evenly distribute this mass to vertices $y$ adjacent to$ u$ such that $y$ is not adjacent to $i$. The
total cost is $\frac{bd_{1}}{a+b}$.
\end{enumerate}
    Therefore, $W_2=\frac{d_3}{a+b}+\frac{ab-a-b}{a(a+b)}d_{2}+\frac{b}{a+b}d_{1} $.

To show the inequality $W_1\geq W_2 \geq W_3$, we first have $W_2\geq \frac{b}{a+b}d_{1}\geq \frac{1}{a}d_{1}\geq \frac{1}{a}d_{3}=W_{3}$. On the other hand,

\begin{equation}
    \begin{split}
        W_1-W_2&=\frac{a-1}{a+b}d_1+\frac{2a}{a+b}d_{2}-\bigg(\frac{d_3}{a+b}+\frac{ab-a-b}{a(a+b)}d_{2}+\frac{b}{a+b}d_{1}\bigg)\\
        &=\frac{a-1-b}{a+b}d_{1}+\frac{2a^2-ab+a+b}{a(a+b)}d_2-\frac{1}{a+b}d_3\\
        &=\frac{a(a-1-b)}{a(a+b)}d_{1}+\frac{2a^2-ab+a+b}{a(a+b)}d_2-\frac{a}{a(a+b)}d_3\\
        &\geq \frac{d_2}{a(a+b)}\bigg[ a(a-b-1)+2a^2-ab+a+b-a\bigg]\\
        &=\geq \frac{d_2}{a(a+b)}\bigg[ a(a-b-1)+a^2-ab+a^2+b\bigg]\geq 0 \quad \text{since $a>b\geq 2$ by assumption}
    \end{split}
\end{equation}

\end{proof}

Now we begin the proof of  Theorem \ref{ricci-flow-detectation}. 

\begin{proof}

Consider a $3\times 3$ matrix $A$ such that

\[
A=\begin{bmatrix}
    \frac{a-1}{a+b} & \frac{2a}{a+b} & 0 \\
    \frac{b}{a+b} & \frac{ab-a-b}{a(a+b)} &\frac{1}{a+b} \\
   0 & 0 & \frac{1}{a}
\end{bmatrix}
\]

For the theoretical recurrence below, we use the unit-step normalized update, i.e., the case $\varepsilon=1$ in \eqref{ricciflow}. Since $\kappa=1-W/d$, this sends each edge length to the corresponding Wasserstein distance before normalization. By Lemma \ref{lemma} and the normalization constraint in Algorithm \ref{algo:ricci_flow}, we have the following constrained system of difference equations:

\begin{equation}
\label{diff system}
        \begin{bmatrix}
        w_{n+1,1}\\
        w_{n+1,2}\\
        w_{n+1,3}
    \end{bmatrix}=A    \begin{bmatrix}
        w_{n,1}\\
        w_{n,2}\\
        w_{n,3}
    \end{bmatrix},\quad w_{n,1}+w_{n,2}+w_{n,3}=|E|
\end{equation}

where $|E|$ is the total number of edges for the given graph. We can rewrite the system of difference equations above as $W_{n+1}=AW_{n}$ by denoting $W_n=[w_{n,1},w_{n,2},w_{n,3}]^{T}$. Here $w_{n,i}$ represents the length of the $i$-th edge type in the graph $G(a;b)$ described in Figure \ref{types of optimal transport} after the $n$-th iteration of the Ricci flow. We use the observation that $d_{1}=w_{n,1}$, $d_{2}=w_{n,2}$, and $d_{3}=w_{n,3}$.

We then introduce another lemma given by \cite{Ni2019CommunityDO}
\begin{lemma} 
\label{lemma2}
Suppose $a>b\geq 2$, there are three real eigenvalues satisfying $\lambda_1>\lambda_{2}=\frac{1}{a}\geq 0 > \lambda_{3}$. In addition, one eigenvector $e_{1}$ associated to $\lambda_{1}$ is of the form $[1,k,0]^{T}$ where $0<k<1$.

\end{lemma}

Suppose the eigenvectors associated to $\lambda_{2}$ and $\lambda_{3}$ are $e_{2}$ and $e_{3}$ respectively, it's easy to see $e_{2}=[0,0,1]^{T}$ for $\lambda_{2}=\frac{1}{a}$. On the other hand, all eigenvalues are distinct, it follows that $e_1,e_2,e_3$ are linearly independent in $\mathbb{R}^{3}$. Under the settings in Algorithm \ref{algo:ricci_flow}, the initial condition becomes $w_{0,1}=w_{0,2}=w_{0,3}=1$.

Note that the general solution for system \eqref{diff system} can now be written as

\[
W_{n}=c_1\lambda_{1}^{n}e_{1}+c_2\lambda_{2}^{n}e_{2}+c_3\lambda_{3}^{n}e_{3}=\begin{bmatrix}
    c_1\lambda_{1}^{n}+o(\lambda_{1}^{n})\\
    kc_1\lambda_{1}^{n}+o(\lambda_{1}^{n})\\
    (\frac{1}{a})^{n}\\
\end{bmatrix}
\]

where $c_{1},c_{2},c_{3}$ can be solved by initial condition and constraint condition. We only keep $c_{1}$ to simplify the expression, which will not affect the conclusion. The small o notation $o(\lambda_{1}^{n})$ represents a function $f(n)=o(\lambda_{1}^{n})$ that grows strictly slower than $o(\lambda_{1}^{n})$ as $n$ approaches infinity, i.e. $\lim_{n\to \infty} o(\lambda_{1}^{n})/\lambda_{1}^{n}=0$.

As a conclusion, we see that the length of 
edge type $uv$ ($w_{n,1}$) grows at the rate of $\lambda_{1}^{n}$ , the length of edge type $ui$ ($w_{n,2}$) and $ij$ ($w_{n,3}$) grows at rate $o(\lambda_{1}^{n})$. More specifically, the length of edge type $ij$ ($w_{n,3}$) shrinks at the rate of $\frac{1}{a}<1$ and shrinks to zero exponentially fast. Namely, the weight of the intra-community edges shrink asymptotically faster than the weight of the inter-community edges.

\end{proof}

\section{Finite-time extension for \texorpdfstring{$\alpha=1/2,p=2$}{alpha=1/2, p=2}}
\label{proof-alpha-half-p-two}

The proof of Theorem \ref{ricci-flow-detectation} above treats the uniform-neighborhood case $\alpha=0,p=0$, where the Wasserstein update reduces to a linear recurrence on the three edge types. The parameter choice used in our experiments, $\alpha=1/2,p=2$, is different: the probability measure at a node depends on the current edge lengths through the exponential term in \eqref{mass dist}. Therefore the same eigenvalue argument does not directly apply. We prove Proposition \ref{prop:practical-separation} by using direct estimates on the nonlinear Wasserstein update.

\begin{proof}[Proof of Proposition \ref{prop:practical-separation}]
Under the ordering $d_1\geq d_2\geq d_3>0$, the direct edge realizes the shortest path between the endpoints of each edge type. Thus the neighbor distances entering \eqref{mass dist} are exactly the corresponding edge-type lengths $d_1,d_2,d_3$.
Write
\[
    E_\ell=\exp(-d_\ell^2),\qquad
    C_g=aE_2+bE_1,\qquad
    C_n=E_2+(a-1)E_3 .
\]
For a gateway node $u$, the mass distribution has mass $1/2$ at $u$, mass $E_2/(2C_g)$ at each of the $a$ non-gateway neighbors in the same community, and mass $E_1/(2C_g)$ at each of the $b$ gateway neighbors in other communities. For a non-gateway node $i$, the distribution has mass $1/2$ at $i$, mass $E_2/(2C_n)$ at its gateway node, and mass $E_3/(2C_n)$ at each of the other $a-1$ non-gateway nodes in the same community.

By the same symmetry and shortest-path transport argument used in Lemma \ref{lemma}, the three Wasserstein distances are
\begin{align}
    W_1
    &=
    \frac{2aE_2+(b-1)E_1}{2C_g}d_1
    +
    \frac{aE_2}{C_g}d_2, \label{eq:w1-alpha-half}\\
    W_2
    &=
    \frac{(a-1)E_3}{2C_n}d_2
    +
    \frac{bE_1}{2C_g}(d_1+d_2), \label{eq:w2-alpha-half}\\
    W_3
    &=
    \frac{E_2+(a-2)E_3}{2C_n}d_3. \label{eq:w3-alpha-half}
\end{align}
For example, in the type-3 case the common mass at the gateway and at the other shared non-gateway neighbors cancels, and only the excess mass at one endpoint must be transported across the edge to the other endpoint. The type-1 and type-2 formulas follow by canceling common mass and then moving the remaining surplus along shortest paths; each displayed plan attains the matching lower bound obtained by counting the mass that must cross the corresponding gateway and within-community edge classes.

Let $X=E_1/E_2$ and $Y=E_3/E_2$. Since $d_1\geq d_2\geq d_3$, we have $0<X\leq 1$ and $Y\geq 1$. From \eqref{eq:w2-alpha-half} and \eqref{eq:w3-alpha-half},
\[
    W_2-W_3
    =
    \frac{(a-1)Yd_2-\bigl(1+(a-2)Y\bigr)d_3}{2(1+(a-1)Y)}
    +
    \frac{bX}{2(a+bX)}(d_1+d_2).
\]
The first term is nonnegative because $d_2\geq d_3$ and $Y\geq 1$, and the second term is strictly positive. Thus $W_2>W_3$.

Similarly, \eqref{eq:w1-alpha-half} and \eqref{eq:w2-alpha-half} give
\[
    W_1-W_2
    =
    \frac{2a-X}{2(a+bX)}d_1
    +
    \left[
    \frac{2a-bX}{2(a+bX)}
    -
    \frac{(a-1)Y}{2(1+(a-1)Y)}
    \right]d_2 .
\]
If the bracketed coefficient is negative, the assumption $d_1\geq d_2$ yields the lower bound
\[
    W_1-W_2
    \geq
    \left[
    \frac{4a-(b+1)X}{2(a+bX)}
    -
    \frac{(a-1)Y}{2(1+(a-1)Y)}
    \right]d_2 .
\]
The first fraction is minimized at $X=1$, while the second fraction is strictly smaller than $1/2$. Hence
\[
    W_1-W_2
    >
    \left[
    \frac{4a-b-1}{2(a+b)}-\frac{1}{2}
    \right]d_2
    =
    \frac{3a-2b-1}{2(a+b)}d_2
    >0,
\]
where the last inequality follows from $a>b\geq 2$. Therefore $W_1>W_2$.

Finally, \eqref{eq:w3-alpha-half} implies
\[
    W_3
    \leq
    \frac{a-1}{2a}d_3,
\]
because the coefficient of $d_3$ is maximized when $Y=1$. Also \eqref{eq:w1-alpha-half} implies
\[
    W_1
    \geq
    \frac{2a+b-1}{2(a+b)}d_1,
\]
because the coefficient of $d_1$ is minimized when $X=1$. Combining the two estimates gives
\[
    \frac{W_3}{W_1}
    \leq
    \frac{(a-1)(a+b)}{a(2a+b-1)}
    \frac{d_3}{d_1}
    =
    q_{a,b}\frac{d_3}{d_1}.
\]
Since $a>b\geq 2$, $q_{a,b}<1$. The normalized Ricci-flow step multiplies all updated edge lengths by the same positive normalization factor, so it preserves both ordering and ratios. Starting from $d_1=d_2=d_3=1$, induction gives the claimed ordering preservation and geometric decay of $w_{n,3}/w_{n,1}$.
\end{proof}

\begin{remark}
Proposition \ref{prop:practical-separation} is intentionally weaker than Theorem \ref{ricci-flow-detectation}. It handles the practical lazy distribution $\alpha=1/2,p=2$, but only proves ordering preservation and relative contraction of the non-gateway intra-community edge under the symmetric $G(a,b)$ model and the unit-step normalized update. It should not be read as a full convergence theorem for arbitrary embedding-derived graphs or for every finite-step implementation detail of Algorithm \ref{algo:ricci_flow}.
\end{remark}

\newpage

\section{Ablation study for MuSiQue}
\label{MuSiQue}

\begin{table*}[h]
  \caption{Ablation studies of Ricci-Filtration on the Accuracy (\%) of MuSiQue dataset}
  \label{ablation-study-musique}
  \begin{center}
  \begin{small}
      \begin{sc}
        \begin{tabular}{l|c|c|c|c|c|c|c|c}
          \toprule
          Methods &$k$         &  $n^{*}$  &$\tau$ &$p$ & $M$ &Embedding model&Reranker model&Acc \\
          \midrule
          Cross-Encoder     &20 & 5 & -& -& -& text-embedding-3-small & bge-reranker-base &\textbf{29.30}\\
          \hline
          Cross-Encoder    & 20 & 10& -& -& -& text-embedding-3-small & bge-reranker-base & \textbf{30.50}\\
          \hline
                    Cross-Encoder     &20 & 5 & -& -& -& text-embedding-ada-002 & bge-reranker-base &\underline{\textbf{30.20}}\\
          \hline
          Cross-Encoder    & 20 & 5& -& -& -& text-embedding-3-small & 
ms-marco-MiniLM-L6-v2 &\underline{\textbf{28.90}} \\
          
          \hline
          Ricci-Filtration    & 10& 6.5& 50\%& 2& 20& text-embedding-3-small & bge-reranker-base &28.10\\
                    \hline
          Ricci-Filtration   & 20& 15.7& 50\%& 2& 20& text-embedding-3-small & bge-reranker-base &28.80\\
          \hline
          Ricci-Filtration    & 20& 15.3& 50\%& 5& 20& text-embedding-3-small & bge-reranker-base &28.50\\
                    \hline
          Ricci-Filtration    & 20& 15.3& 50\%& 1& 20& text-embedding-3-small & bge-reranker-base &27.10\\
          \hline
            Ricci-Filtration   & 20& 16.6& 25\%& 2& 20& text-embedding-3-small & bge-reranker-base &\textbf{29.20}\\
          \hline
          Ricci-Filtration    & 20& 14.4& 75\%& 2& 20& text-embedding-3-small & bge-reranker-base &26.20\\        
                    \hline
            Ricci-Filtration   & 20& 14.6& 50\%& 2& 10& text-embedding-3-small & bge-reranker-base &28.30\\
          \hline
          Ricci-Filtration    & 20& 10.2& 50\%& 2& 1& text-embedding-3-small & bge-reranker-base &23.20\\  
          \hline
                      Ricci-Filtration   & 20& 16.1& 50\%& 2& 20& text-embedding-ada-002 & bge-reranker-base &\underline{26.40}\\
          \hline
          Ricci-Filtration    & 20& 16.4&50\%& 2& 20& text-embedding-3-small & ms-marco-MiniLM-L6-v2 &\underline{28.60}\\ 
          \hline
        \end{tabular}
      \end{sc}
      \end{small}
  \end{center}
  {\scriptsize\noindent\emph{Note.} The better performing values are highlighted in bold. The results from different reranker models are highlighted with underline. $n$ for Ricci-Filtration methods represents the average number of chunks fed into reranker model after Ricci-Filtration.\par}
\end{table*}

\section{K-means filtering ablation}
\label{kmeans-ablation}

To separate the effect of Ricci-flow-based graph filtration from a simpler clustering heuristic, we compare Ricci-Filtration with a K-means filtering baseline. The K-means baseline clusters the query and initially retrieved chunk embeddings into two clusters ($k=2$), keeps the cluster containing the query node, and then follows the same downstream RAG procedure. For Ricci-Filtration, we use $\alpha=0.5$ and $p=2$; both methods use the same initial retrieval setting and gpt-4o-mini generation setting.

\begin{table*}[h]
  \caption{Accuracy (\%) comparison between Ricci-Filtration and K-means filtering with $k=2$.}
  \label{kmeans-filtering-table}
  \begin{center}
    \begin{small}
      \resizebox{\linewidth}{!}{
      \begin{tabular}{lcccccc}
        \toprule
        Methods & SQuADv2 & HotpotQA & MultiHop Inference & MultiHop Comparison & MultiHop Temporal & MultiHop Null \\
        \midrule
        Ricci-Filtration & \textbf{73.00} & 67.00 & \textbf{97.18} & 64.02 & \textbf{56.60} & \textbf{95.00} \\
        K-means filtering ($k=2$) & 63.70 & \textbf{67.70} & 88.48 & \textbf{67.99} & 50.77 & 94.02 \\
        \bottomrule
      \end{tabular}
      }
    \end{small}
  \end{center}
\end{table*}

The comparison is mixed but informative. Ricci-Filtration improves four of the six reported accuracy metrics, including SQuADv2 and three MultiHop-RAG query types. K-means filtering is slightly better on HotpotQA and MultiHop comparison queries. Thus, the result supports the usefulness of the graph-geometric filtration step over a simple Euclidean clustering baseline in several settings, while also reinforcing that Ricci-Filtration is not uniformly superior on harder multi-hop cases.

\section{LLM-based reranker prompt}
\label{llm-reranker-prompt}

For the LLM-based reranker baseline, each query is paired with the initially retrieved passages, formatted as \texttt{Document i: ...}. The model is instructed to return a JSON dictionary mapping each passage id to a relevance score from 0 to 10. The system prompt and user-message template are shown below.

\paragraph{System prompt.}
\begingroup
\footnotesize
\begin{verbatim}
You are a customer support answer service. Your task is to evaluate help
center passages and score their relevance to a given customer query for a
retrieval augmented generation (RAG) system.

Evaluation Process:
1. Analyze the customer's query to identify both explicit needs and implicit
   context, including underlying user goals.
2. Assess each passage's ability to directly resolve the query or provide
   substantive supporting information with actionable guidance.
3. Score based on how effectively the passage addresses the query's core
   intent while considering potential interpretations.

Grading Criteria:
<grading_scale>
10: EXCEPTIONAL match - Contains exact step-by-step instructions that
    perfectly match the query's specific scenario. Must include all required
    parameters/context and resolve the issue completely without any ambiguity.
    Reserved for definitive solutions that exactly mirror the user's
    described situation and require no interpretation.

9: NEAR-PERFECT solution - Contains all critical steps for resolution but may
   lack one minor non-essential detail. Addresses the precise query parameters
   with specialized information. Solution must be directly applicable without
   requiring adaptation or assumptions.

8: STRONG MATCH - Provides complete technical resolution through specific
   instructions, but may require simple logical inferences for full
   application. Covers all essential components but might need minor
   contextualization.

7: GOOD MATCH - Contains substantial relevant details that address core
   aspects of the query, but lacks one important element for complete
   resolution. Provides concrete guidance requiring some user interpretation.

6: PARTIAL match - General guidance on the right topic but lacks the specifics
   for direct application. May only resolve a subset of the request.

5: LIMITED relevance - Related context or approach, but indirect. Requires
   substantial effort to adapt to the user's exact need.

4: TANGENTIAL - Mentions related concepts/keywords with little practical
   connection to the request. Minimal actionable value.

3: VAGUE domain info - Talks about the general area but not the query's
   specifics. No concrete, actionable steps.

2: TOKEN overlap - Shares isolated terms without context or intent aligned to
   the request. Similarity is coincidental.

1: IRRELEVANT - Uses query terms in a completely unrelated way. No meaningful
   link to the user's goal.

0: UNRELATED - No thematic or contextual connection to the query at all.
</grading_scale>
\end{verbatim}
\endgroup

\paragraph{User message template.}
\begingroup
\footnotesize
\begin{verbatim}
Question:
{query}

Context:
Document 0: {passage_0}

Document 1: {passage_1}

...

Output Format:
<output_format>
Return your response in a valid JSON (skip spaces):
{"id0":score0,"id1":score1,...}

Strict guidelines:
- Return ONLY a well-formed valid JSON with passage IDs as keys
- Each key must be a passage id in the format "idN"
- Integer values only, no decimals
- Skip spaces in the JSON
- No additional text or formatting
- Maintain original passage ID order
</output_format>

Reranked documents:
\end{verbatim}
\endgroup

\section{Time efficiency comparison}
\label{time efficiency tab}
\begin{table}[!htp]
  \caption{Average and standard deviation running time (in seconds) for each QA task under different methods.  }
  \label{time-table}
  \begin{center}
    \begin{small}
      \begin{sc}
        \begin{tabular}{lccc}
          \toprule
          Methods & SQUADv2         & HotpotQA   &  MuSiQue   \\
          \midrule
          Cross-Encoder Reranker     & 1.08$\pm$ 0.22 &  1.37 $\pm$ 0.30 & 4.18 $\pm$ 0.88\\
          LLM as Reranker  & 3.75$\pm$ 0.22 &4.82$\pm$ 0.95 & 4.25 $\pm$ 1.14\\
          Ricci-Filtration    & 8.62$\pm$ 0.32 &9.41 $\pm$ 0.47& 12.37 $\pm$ 1.21 \\
          \bottomrule
        \end{tabular}
      \end{sc}
    \end{small}
  \end{center}

\end{table}

The average and standard deviation running time is calculated based on 10 randomly pick questions 

\end{document}